\documentclass[journal]{IEEEtran}
\ifCLASSOPTIONcompsoc
  \usepackage[nocompress]{cite}
\else
  \usepackage{cite}
\fi
\usepackage{amsmath}
\usepackage{amssymb}
\usepackage{amsfonts}
\usepackage{enumerate}
\usepackage{epsfig}
\usepackage{graphicx}
\usepackage{balance}

\usepackage[noend]{algpseudocode}
\usepackage{algorithmicx,algorithm}

\def\ie{{\em i.e.}}
\def\eg{{\em e.g.}}
\def\etal{{\em et al.}}

\newcommand{\figref}[1]{Fig. \ref{#1}}
\newcommand{\tabref}[1]{Tab. \ref{#1}}

\newcommand{\secref}[1]{Section \ref{#1}}

\newcommand{\mc}[1]{\mathcal{#1}}

\newcommand{\tabincell}[2]{\begin{tabular}{@{}#1@{}}#2\end{tabular}}

\newcommand{\br}[1]{\bm{\mathrm{#1}}}
\newcommand{\bs}[1]{\boldsymbol{\texttt{#1}}}

\newtheorem{proof}{Proof}[section]

\graphicspath{{./figure/}}

\usepackage{ragged2e}
\usepackage{booktabs}
\usepackage{color}
\usepackage{multirow}
\usepackage{bm}
\usepackage{bbm}
\usepackage{url}

\begin{document}
\title{Part-guided Relational Transformers for Fine-grained Visual Recognition}

\author{Yifan~Zhao,~Jia~Li,~\IEEEmembership{Senior Member,~IEEE}, Xiaowu~Chen,~\IEEEmembership{Senior Member,~IEEE}, \\ and~Yonghong~Tian,~\IEEEmembership{Senior Member,~IEEE}% <-this % stops a space
\IEEEcompsocitemizethanks{\IEEEcompsocthanksitem Y. Zhao, J. Li, and X. Chen are with the State Key Laboratory of Virtual Reality Technology and Systems, School of Computer Science and Engineering, Beihang University, Beijing, 100191, China.
\IEEEcompsocthanksitem Y. Tian is with the Department of Computer Science and Technology, Peking University, Beijing, 100871, China.
\IEEEcompsocthanksitem J. Li, X. Chen and Y. Tian are also with the Pengcheng Laboratory, Shenzhen, 518055, China.
\IEEEcompsocthanksitem J. Li is the corresponding author (E-mail: jiali@buaa.edu.cn).
}}

\markboth{ IEEE TRANSACTIONS ON IMAGE PROCESSING}%
{Zhao \MakeLowercase{\textit{et al.}}: IEEE TRANSACTIONS ON IMAGE PROCESSING}

\maketitle

\begin{abstract}
Fine-grained visual recognition is to classify objects with visually similar appearances into subcategories, which has made great progress with the development of deep CNNs. However, handling subtle differences between different subcategories still remains a challenge. In this paper, we propose to solve this issue in one unified framework from two aspects,~\ie, constructing feature-level interrelationships, and capturing part-level discriminative features. This framework, namely PArt-guided Relational Transformers (PART), is proposed to learn the discriminative part features with an automatic part discovery module, and to explore the intrinsic correlations with a feature transformation module by adapting the Transformer models from the field of natural language processing. The part discovery module efficiently discovers the discriminative regions which are highly-corresponded to the gradient descent procedure. Then the second feature transformation module builds correlations within the global embedding and multiple part embedding, enhancing spatial interactions among semantic pixels.
Moreover, our proposed approach does not rely on additional part branches in the inference time and reaches state-of-the-art performance on 3 widely-used fine-grained object recognition benchmarks. Experimental results and explainable visualizations demonstrate the effectiveness of our proposed approach. Code can be found at {\color{blue}\url{https://github.com/iCVTEAM/PART}}.

\end{abstract}


% \fi
\IEEEpeerreviewmaketitle

\section{Introduction}\label{sec:intro}

\IEEEPARstart{F}{ine}-grained visual recognition aims to classify and distinguish the subtle differences among sub-categories with similar appearances, which has made great progress in recognizing an increasing number of categories. Benefiting from the development of Convolutional Neural Networks (CNNs), performance on recognizing common object categories,~\ie, birds~\cite{wah2011caltech,van2015building}, cars~\cite{krause20133d}, and aircrafts~\cite{maji2013fine}, has increased steadily in the last decade. However, further explorations on discriminative features and representative feature embedding still face great challenges, which also limit the performance improvement on fine-grained recognition tasks.

Existing methods in solving this challenge can be roughly divided into two predominant groups, considering the different learning manners of fine-grained features. The first group tackles the fine-grained classification problem by generating rich feature representations~\cite{lin2015bilinear,wang2018learning,zhang2019learning,zheng2019looking,gao2020channel} or applying auxiliary constraints~\cite{sun2018multi,dubey2018maximum}. Leading by the bilinear pooling operation~\cite{lin2015bilinear}, second-order matrices across different channels are widely adopted by introducing compact homogeneous representations~\cite{gao2016compact}, hierarchical organizations~\cite{yu2018hierarchical} and other dimensional reduction operations~\cite{kong2017low,li2017factorized,wei2018grassmann}. Second-order pooling operations describe the fine-grained object features with rich pair-wise correlations between network channels, which are regarded as high-dimensional descriptors for discovering fine-grained features.
Beyond the exploitation of second-order matrices, trilinear attention across semantic channels~\cite{zheng2019looking,gao2020channel} are proposed to build global attention as well as maintaining the feature shapes. However, there still remain two major problems in this group of methods: 1) limited by the natural flaws of CNNs, the constructed correlations are still conducted locally and restricted in channel dimensions, and the long-term spatial relationships are still unperceived; 2) methods of the second group focus on the constraints on feature learning, but neglect the discriminative part features, which are necessary for distinguishing near-duplicated samples.

\begin{figure}[!t]
	\centering
	\includegraphics[width=\columnwidth]{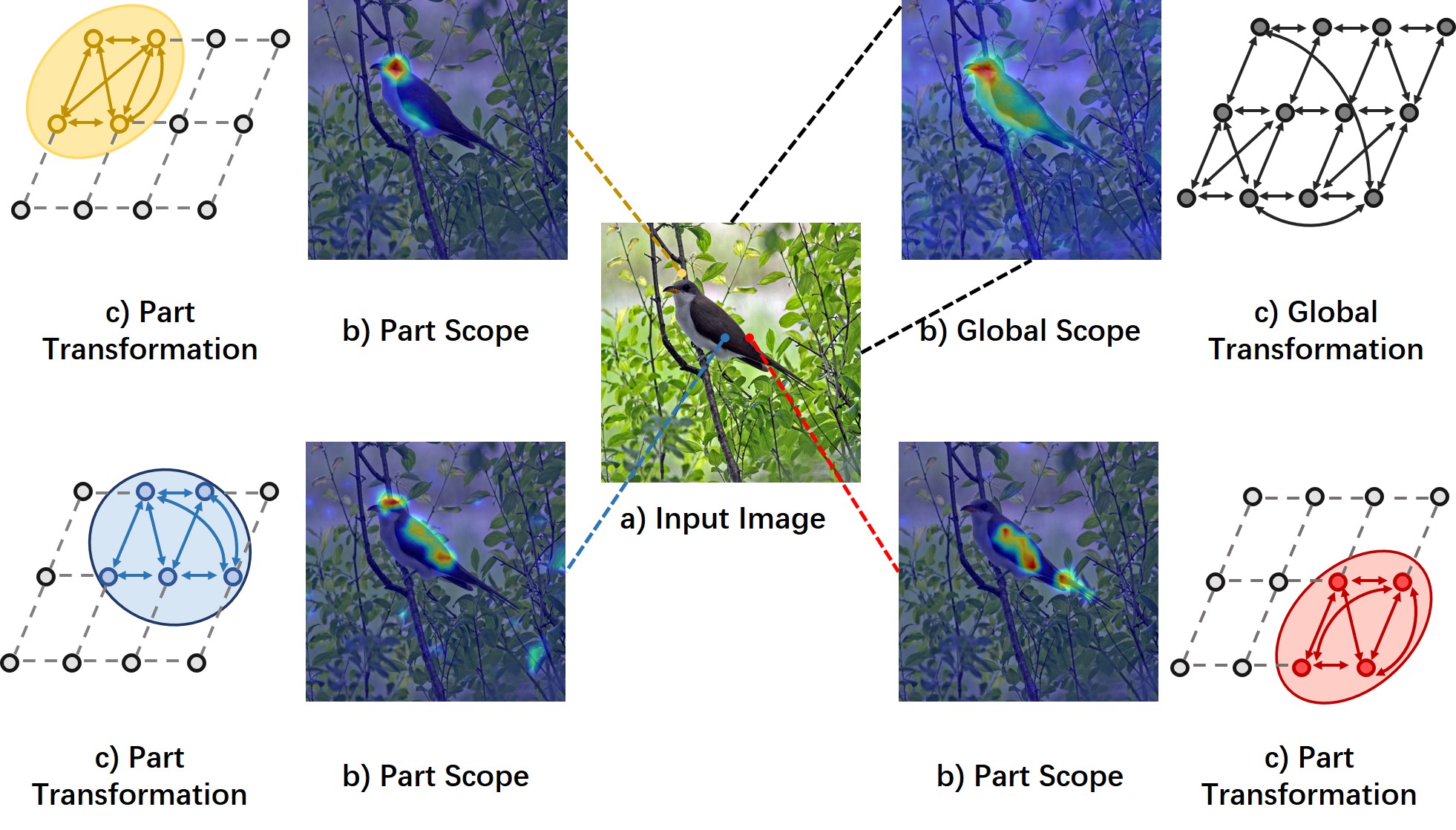}
	\caption{Motivation of the proposed approach. a): input image. b): one global feature scope and three automatic part discovery scopes. c): relation transformation module to construct global relation embedding and local relational embedding. Our framework first discovers the discriminative local regions and then constructs contextual interactions with relation transformations.
	}
	\label{fig:motivation}
\end{figure}

To tackle the second problem, methods of the other group propose to localize distinct parts for feature enhancement. Pioneer methods tend to obtain object parts by using part detectors~\cite{huang2016part,ge2019weakly} or segmentation parsers~\cite{kalayeh2018human,huang2020interpretable}. Whilst promising results have been achieved by introducing additional network branches, annotating part segmentation masks or bounding boxes is labor-consuming. Toward this issue, recent approaches~\cite{fu2017look,recasens2018learning,yang2018learning} resort to attention mechanisms for exploring auto-emerged parts during the backpropagation process. For example, Fu~\etal~\cite{fu2017look} proposed a zoom-in strategy to discover the most informative region in a progressive manner.
Nevertheless, auto-discovered object parts are usually unstable, which would lead to overfitting on local regions. Hence a contextual understanding of global relationships is needed in solving fine-grained classification problems.

To efficiently solve the deficiencies of these two groups of approaches, we propose to learn the discriminative part features and intrinsic correlations in one unified framework, namely PArt-guided Relational Transformers (PART). To be specific, our proposed PART is composed of two essential modules,~\ie, a part discovery module to investigate discriminative local regions, and a relational transformation module to construct local and contextual correlations. Inspired by the success in building long-term dependencies in Natural Language Processing (NLP), we make an attempt to learn the relationships among high-level features by using Transformers~\cite{vaswani2017attention}. In this paper, we embed the Transformer model~\cite{vaswani2017attention} into the natural design of representative CNN architectures, incorporating self-related global relationships of network features with the stack of transformation layers. The high-level feature maps after CNN encoders are decomposed as several individual vectors by their spatial dimensions. After that, each spatial region has the potential to perceive the contextual information of any other region and builds semantic correlations with each other, which help the networks understand the holistic objects rather than restricted in local patterns. This novel network architecture greatly broadens the receptive field of conventional CNN architectures, while simultaneously keeps their spatial distributions by a positional-aware encoding.

As shown in~\figref{fig:motivation}, our proposed approach first investigates the discriminative regions by generated Class Activation Maps (CAMs)~\cite{zhou2016learning}, exploring several part scopes and one global scope in~\figref{fig:motivation} b). These part regions are located by an automatic part discovery module, which maximizes the diversity of different semantic parts while finding discriminative regions simultaneously. With these scopes for object and part discovery, we further propose to learn the local relational embedding and global contextual embedding in~\figref{fig:motivation} c) with masked multi-head attentions. Hence not only the global relationships are explored in our approach, but also the locally-connected relations. With these two insightful modules, our approach can well embed the long-term dependencies and generate robust part localizations for the accurate classification of fine-grained species. It is worth noting that our proposed PART approach does not rely on the additional part branches in the inference time, which introduces less computation cost. Moreover, our proposed approach does not use additional data annotations or multi-crop testing operations, and achieves state-of-the-art performances on three public fine-grained recognition benchmarks,~\ie, CUB-200-2011~\cite{wah2011caltech}, Stanford-Cars~\cite{krause20133d}, and FGVC-Aircrafts~\cite{maji2013fine}. Further experimental results and visualized explanations reveal the effectiveness of the two proposed modules.

In summary, the contribution of this paper is three-fold:
\begin{enumerate}
\item We propose a novel PArt-guided Relational Transformers (PART) framework for fine-grained recognition tasks, which broadens the long-term dependencies of conventional CNNs with a part-aware embedding. To the best of our knowledge, it is the first attempt to handle fine-grained visual recognition tasks with joint CNN and Transformer architectures.
\item We propose a new part discovery module to mine the discriminative regional features, to robustly handle differential part proposals by class activation maps.
\item We propose a part-guided relational transformation module to learn intrinsic high-order correlations and conduct comprehensive experiments to verify the effectiveness of our proposed approach, surpassing state-of-the-art models by a margin on 3 widely-used benchmarks.
\end{enumerate}

The remainder of this paper is organized as follows:~\secref{sec:relatedwork} briefly reviews related works of fine-grained recognition tasks and~\secref{sec:method} describes how to discover distinct part features and explore the part-guided contextual relationship in our proposed approach. Qualitative and quantitative experiments are reported in~\secref{sec:experiment}.~\secref{sec:conclusion} finally concludes this paper.

\section{Related Work}\label{sec:relatedwork}
Early fine-grained visual recognition researches~\cite{yao2011combining,yao2012codebook,lazebnik2006beyond} focused on tackling limited labeling data with hand-crafted conventional features. For example, Yao~\etal~\cite{yao2011combining} proposed to use dense sampling strategies and random forest for mining discriminative features. Moreover, methods~\cite{yao2012codebook} using these hand-crafted features show its benefits considering the independence of human labeling or encoded dictionary features.
Recent ideas in solving fine-grained visual recognition tasks focus on the exploitation of deep features, which show both robustness and performance boost for handling discriminative features. Considering the different constraints of representation learning, we mainly consider two families of methods in this paper,~\ie, part awareness learning and regularized feature representation.

\subsection{Part awareness learning} Handling subtle differences using part-level local features has been widely studied in the past decades. Pioneer works~\cite{zhang2014part,huang2016part,he2017weakly,wei2018mask,he2019part,huang2020interpretable} tend to discover the part-level features with manual labels and amplify these local representations when obtaining the final features. For example, Zhang~\etal~\cite{zhang2014part} proposed a part-based R-CNN network, using part detectors for building pose-normalized representations. Krause~\etal~\cite{krause2015fine} proposed to detect the object part with an unsupervised co-segmentation and alignment manner, enhancing the final classification abilities.
Huang~\etal~\cite{huang2016part} integrated a part detection network in the recognition network and developed two streams to encode the part-level and object-level cues simultaneously.
In addition, Wei~\etal~\cite{wei2018mask} proposed to learn accurate object part masks using an additional FCN and then aggregated these masked features with the global ones. Although promising results have been achieved by using strongly supervised annotations, labeling segmentation masks and accurate object parts are still time-and labor-consuming.

On the other hand, deep neural networks have the natural ability in discovering object parts~\cite{gonzalez2018semantic} with only weakly supervised labels. Object parts would emerge automatically in the feature maps during the gradient learning process. Benefiting from these observations, recent works~\cite{xiao2015application,simon2015neural,lam2017fine,fu2017look,recasens2018learning,yang2018learning,ding2019selective,wang2020weakly} proposed to learn the part-aware features with auto-discovered part masks.
Leading by this motivation, Simon~\etal~\cite{simon2015neural} proposed a local region discovery model using deep neural activation maps, constructing part-level attentions for final representation. Lam~\etal~\cite{lam2017fine} proposed to search object parts using heuristic function and successor functions.
Wei~\etal~\cite{wei2017selective} proposed a flood-fill method which aggregated these features maps as holistic object activations.
However, this operation only considers the extraction of object features, which does not reflect the part localizations.
Peng~\etal~\cite{peng2017object} proposed to select image patches as candidates to select the informative regions and extract the feature of these regions for complementation. Wang~\etal~\cite{wang2020weakly} proposed to regularize the learned feature maps in a sparse representation manner and adopted axillary part branches for feature representation regularization.
Moreover, Ge~\etal~\cite{ge2019weakly} proposed a bottom-up architecture to obtain instance detection and segmentation via weakly supervised object labels. Combing these part proposals, a context encoding LSTM is further proposed to fuze the sequential part information for image classification.
However, these works only considered the localization of part regions, while the contextual relationship of these local regions is less explored. In this paper, we propose to discover the discriminative parts and investigate their spatial relations in one unified framework.

\subsection{Regularized feature representation} The other family for solving the fine-grained visual categorization task is to regularize the feature learning process or formulating rich feature relationships. Inspired by the bilinear pooling operation~\cite{lin2015bilinear}, high-order relationship matrices~\cite{yu2018hierarchical,zhang2019learning,zheng2019looking,zhao2021graph,du2020fine} has been widely used for feature representations. For example, Yu~\etal~\cite{yu2018hierarchical} proposed to learn the heterogonous feature interaction from different levels of feature layers, building compact cross-layer relationships.
However, simply adopting the second-order feature would lead to a dimensional disaster for optimization, thus subsequent works proposed to use low-rank presentations~\cite{kong2017low}, feature factorization~\cite{li2017factorized} or Grassmann constraints~\cite{wei2018grassmann} for building compact homogenous features.
These works~\cite{kong2017low,li2017factorized,wei2018grassmann} greatly reduce the computation cost of high order matrices but still show comparable results.
Besides the efforts on the second-order relationships, the recent network tends to explore the third-order relationships of feature channels. Leading by the non-local~\cite{wang2018non} design in building global relationships of different channels, several works~\cite{zheng2019looking,gao2020channel} have been proposed to explore the channel-wise relationship by attention mechanisms. Gao~\etal~\cite{gao2020channel} proposed to learn the channel-wise relationship by image-level relations and cross-image relations by a contrastive constraint.

Thinking the feature representation problem from another perspective, other works tend to explore the cross-layer relationships~\cite{chang2020devil,zheng2017learning,sun2018multi,luo2019cross,wang2018learning,ji2020attention} or auxiliary information~\cite{aodha19presence}.
Sun~\etal~\cite{sun2018multi} proposed a multi-attention multi-constraint (MAMC) to regularize the same class focusing on the same attention region.
Based on this MAMC module, Luo~\etal~\cite{luo2019cross} proposed to learn relationships between different images and different network layers.
In addition, Chen~\etal~\cite{chen2019destruction} proposed a destruction and reconstruction framework to learn the local feature transformation rules for obtaining robust feature representations. Moreover, an attentive pair relation network~\cite{zhang2020learning} is proposed to fine the group-wise relationship inter-and intra-classes.
Beyond these designs on network architectures, adopting pair-wise confusions~\cite{dubey2018pairwise} and maximizing the entropy~\cite{dubey2018maximum} also show effectiveness in improving the final representation.
Different from the aforementioned methods, in this work, we propose to find the feature embedding by two steps,~\ie, discovering local parts and transforming for relational embedding, which incorporates the advantages of part discovery but also constructs reliable feature regularization.

\subsection{Vision Transformers} Transformer architectures~\cite{vaswani2017attention,choromanski2020rethinking} have achieved great progress in the field of natural language processing. Benefiting from its ability for long-term relationship construction, several works~\cite{carion2020end,zhu2020deformable} proposed to introduce this architecture based on the convention CNN encoders.
For example, Carion~etal~\cite{carion2020end} proposed to localize different object regions with multiple attention regions in transformer layers. Beyond these works, recent ideas~\cite{dosovitskiy2020image,liu2021swin,yuan2021tokens,srinivas2021bottleneck} proposed to directly encode the images patches,~\eg, $16\times 16$ as the basic pattern to understanding images. As the representative work, ViT~\cite{dosovitskiy2020image} built a strong pretrained transformer backbone, which greatly surpass the performance of conventional CNNs. However, this encoding manner also introduces huge computation costs. Several works proposed to introduce the shifted windows~\cite{liu2021swin} and local encoding~\cite{yuan2021tokens} to build hierarchical encoders, which exploit the advantages of local understandings in CNNs and also reduce computation costs. Based on these successful backbone encoders, utilizing vision transformers has also shown great benefits in semantic segmentation~\cite{xie2021segformer} and fine-grained visual classification~\cite{he2021transfg}.
Different from the TransFG~\cite{he2021transfg} with ViT~\cite{dosovitskiy2020image} encoders, our model builds relationships between high-level CNN features. In this manner, the local patterns with detailed features can be well represented and the contextual relationships are also constructed simultaneously.

\begin{figure*}[!t]
\begin{center}
\includegraphics[width=1\textwidth]{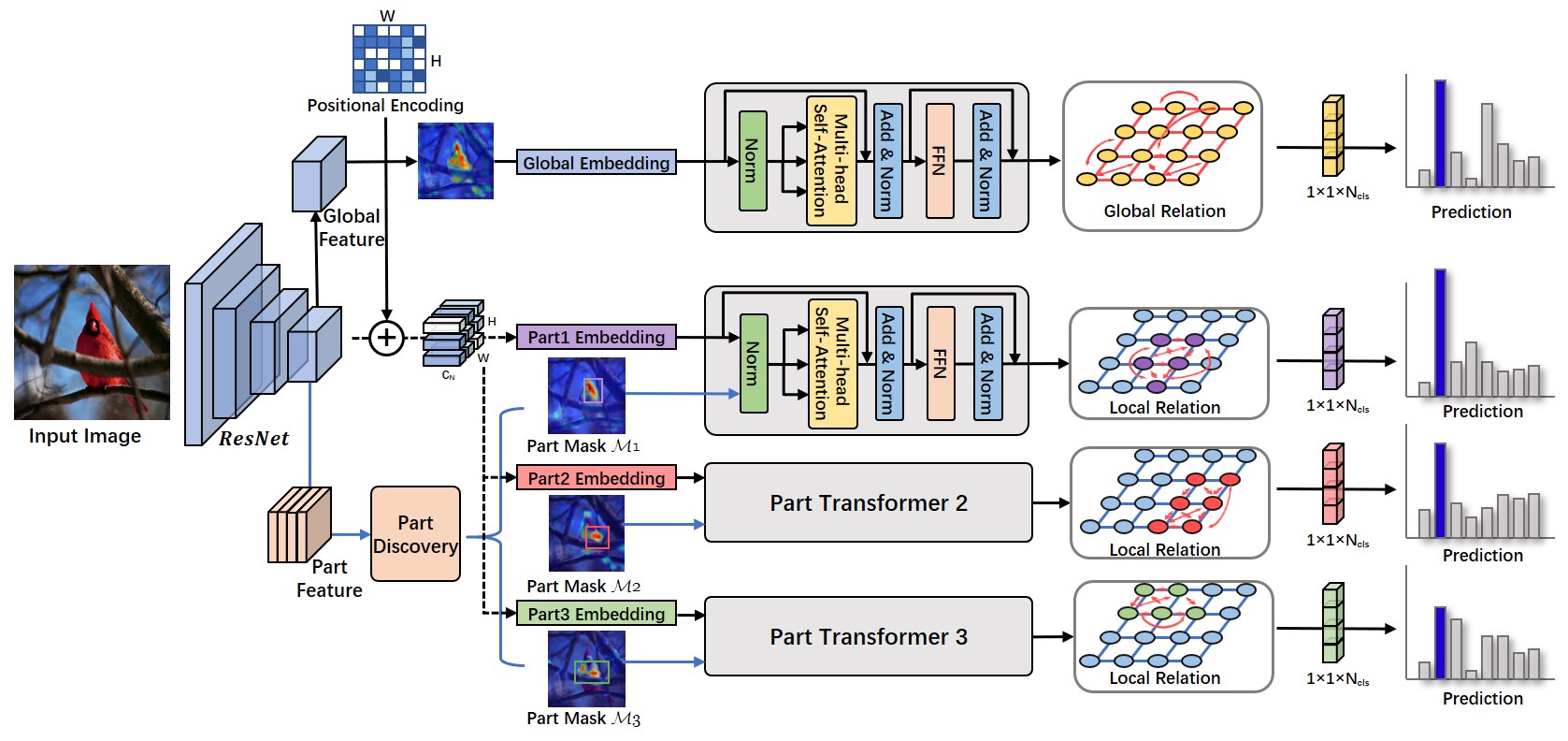}
 \caption{Overall pipeline of our proposed PArt-guided Relational Transformers (PART) framework, which is composed of a part discovery module and a relational transformation module. The part discovery module automatically generates object parts from class activation maps during the forward propagation process. And the relational transformation module takes advantages of object parts and constructs a global interaction and S (S=3 for example) local interactions. During the inference stage, only the global branch is maintained, introducing less computation cost.}
 \label{fig:pipeline}
 \end{center}
\end{figure*}

\section{Approach}\label{sec:method}
\subsection{Problem Formulation}\label{sec:formulate}
Our proposed PART approach for fine-grained visual recognition is illustrated in~\figref{fig:pipeline}. Our first key idea in exploring fine-grained recognition is to build contextual dependencies by the conjunction of successful Transformers with conventional Convolutional Neural Networks (CNNs). The second motif is to regularize the network embedding be aware of the local discriminative regions, which is desired for distinguishing visually similar samples. Beyond these designs, one global relation transformation and three local relation transformations are proposed to learn the crucial attention features by building regionally dense connections during the gradient descent process.

Given an input image $\mc{I}$, the general procedure in optimizing fine-grained recognition task is to optimize the target $\mathbb{E}$ between prediction and groundtruth label $\br{y}$:
\begin{equation}\label{eq:globalform}
\br{\theta}^{*}=\arg\min_{\br{\theta}} \mathbb{E}(\Phi(\mc{I};\br{\theta})^{\top}\br{w}_c, \br{y}),
\end{equation}
where $\Phi(\cdot;\br{\theta})$ represents feature extractor $\Phi$ with learnable parameters $\br{\theta}$. $\br{w}_c$ represents the final classification projection vectors. Different from the general classification optimizations, our proposed network first generates two different feature embedding,~\ie, $\br{X}^p,\br{X}^g \in \mathbb{R}^{W \times H \times C}$. Thus we could localize the object part proposals $\br{P}$ with part discovery module $\mc{M}_{part}$,~\ie, $\mc{P}=\mc{M}_{part}(\br{X}^p)$, and construct high-order correlations with Transformers $\mc{G},\mc{T}$ for global and local branches.

Beyond these foundations, we thus make a general assumption that the object part proposals $\mc{P}$ share the same groundtruth labels $\br{y}$ with the global image $\mc{I}$. Hence the overall optimization in our framework can be presented as:
\begin{equation}\label{eq:partform}
\begin{split}
\br{\theta}^{*}=\arg\min_{\br{\theta}} \mathbb{E}((\Phi(\mc{I};\br{\theta}_b) \circledast \mc{G}(\br{X}^g;\br{\theta}_{g}) )^{\top}\br{w}_g, \br{y})+ \\
 \sum_{\br{p} \in \mc{P} } \lambda_{\br{p}} \mathbb{E}((\Phi(\mc{I};\br{\theta}_b) \circledast \mc{T}^{\br{p}}(\br{p},\br{X}^p;\br{\theta}_{\br{p}})) ^{\top}  \br{w}_l, \br{y}) ,
\end{split}
\end{equation}
where $\circledast$ indicates that the latter term extract features from the former one. $\br{w}^{ \{ g,l \} }$ represent the global and local classification projections and $\lambda$ is the balanced weight.
Hence our network introduces the relational transformation constraints and part localization constraints in fine-grained classification tasks. We will elaborate the details of these two modules in~\secref{sec:partdis} and~\secref{sec:relationembed}.

\subsection{Part Discovery Module}\label{sec:partdis}
Excavating discriminative features by localizing object parts has been widely explored by previous works~\cite{zhang2014part,huang2016part,he2017weakly,wei2018mask,he2019part,huang2020interpretable} with bounding box annotations or segmentation mask annotations. In our proposed PART framework, we propose to explore the object parts in a weakly supervised manner, which takes advantage of the class activation maps generated during the gradient descent process. Recent research~\cite{gonzalez2018semantic} shows that object parts emerge automatically during the learning of classification tasks,~\ie, during the training process, network features can be automatically regularized for excavating discriminative features and distinguishing near-duplicated objects.

Based on these findings, several recent ideas, \eg, RA-CNN~\cite{fu2017look}, MA-CNN~\cite{zheng2017learning}, propose to constrain these activation regions into part groups by their similarities. Different from these works, our part discovery module aims to provide discriminative region priors $\mc{P}$ for the transformer module. With these priors, the transformer module is proposed to mine the contextual object clues within each region. To achieve this, our proposed part discovery module follows two meaningful rules: 1) the highly activated features maps are attached with high priorities during the back-propagation process; 2) different activated parts should have the minimum overlap for discriminative part finding.  These rules ensure CAMs are not restricted to local regions and enhance the generalization ability of deep models. Besides that, our part discovery module can be trained end-to-end in one unified scheme and does not cost additional computation costs for inference.

Given the extracted backbone features $\br{X}^p \in \mathbb{R}^{W \times H \times C}$, we define the class activation scores $\br{s}_c$ of the $c$th channel:
\begin{equation}\label{eq:partscore}
\begin{split}
\br{v}_c&=\frac{1}{W\times H} \sum_{i=1}^H \sum_{j=1}^W \br{X}^{p}_{i,j,c},\\
\br{s}_c&= \frac{\br{v}_c} {\sum_{k=1}^{C}(\br{v}_k+\epsilon)},
\end{split}
\end{equation}
where $\epsilon$ is set to keep nonzero of the denominator. With this score as the ranking indicator, we thus rearrange the $\br{X}^{*}  = \bs{ActivationSort}(\br{X}$, $\br{s})$ and build a part proposal stack to select the most important one from the stack top.

Starting from this part proposal stack, we first sample a random rate for each part $\eta_c \sim  {N}(\mu,\sigma)$, enhancing robustness for threshold selection. For each selected feature $\br{x}^{*}_{c} \in \br{X}$, the part proposals $\br{p}^{*}$ are localized with the $\bs{ROICrop}$ operation:
\begin{equation}\label{eq:croppart}
\br{p}^{*}_{i,j,c} =
\left\{\begin{matrix}
 1&   \frac{\br{x}^{*}_{i,j,c}-\min (\br{x}^{*}_{c}) }{ \max(\br{x}^{*}_{c})-\min (\br{x}^{*}_{c}) } \ge \eta_c, \\
 0&   \frac{\br{x}^{*}_{i,j,c}-\min (\br{x}^{*}_{c}) }{ \max(\br{x}^{*}_{c})-\min (\br{x}^{*}_{c}) } < \eta_c,
\end{matrix}\right.
\end{equation}

\begin{algorithm}[!t]
\caption{Part Discovery Algorithm}\label{alg:part}
\hspace*{0.02in} {\bf Input:}
Part feature maps $\br{X} \in \mathbb{R}^{W \times H \times C}$, Hyperparameters: Number of parts $N$, Selection range $R$, IoU threshold $th$\\
\hspace*{0.02in} {\bf Output:}
Part Selection Set:  $\mc{P} = \{\br{p}_1,\br{p}_2, \ldots, \br{p}_N\}$
\begin{algorithmic}[1]
\State Initialize selection set $\mc{P}=\phi$
\While{ $|\mc{P}|<N$ }
    \State Calculate part activation scores $\br{s}$ of $\br{X}$  in Eqn.~\eqref{eq:partscore}
    \State Rearrange $\br{X}$ by the ranking of $\br{s}$:
    \Statex $\qquad   \br{X}^{*}  = \bs{ActivationSort}(\br{X}$, $\br{s})$
    \For{ $ \br{x}^{*}_i \in \br{X}^{*} $}
    \State Sample Random Activation Rate: $\eta_i \sim  {N}(\mu,\sigma)$
    \State $\widetilde{\br{x}_i} =\bs{Norm}(\br{x}^{*}_i) $
    \State $\br{p}^{*}_i= \bs{ROICrop} (\widetilde{\br{x}_i} \ge \eta_i)$:
    \For{ $ \br{p}_j \in \mc{P}$}
    \State Calculate bounding box IoU:
    \Statex \qquad \qquad \qquad  BIoU$=\bs{BboxIoU}(\br{p}_j,\br{p}^{*}_i)$
    \If{BIoU $> th$}
    \State break
    \EndIf
    \State Add new legal proposals: $\mc{P} = \mc{P} \cup \{ \br{p}^{*}_i \} $
    \EndFor
    \EndFor
    \State iter = iter+1
    \If{iter $>$ maxiter and $|\mc{P}|<N$}
	\State Select top $N$ parts from $\br{X}^{*}$ and repeat Step $6\sim 8$
    \EndIf
\EndWhile
\State \Return Part Selection Set $\mc{P}$
\end{algorithmic}
\end{algorithm}

The final bounding box regions can be automatically obtained by calculating the range of coordinates in the x-axis and y-axis.
However, as an unsupervised part discovery method, different feature maps usually localize similar part regions, which may restrict the learning of discriminative features. Inspired by the Non-Maximum Suppression operation~\cite{rosenfeld1971edge} in edge detection, we remove the redundant part proposals by calculating the Intersection over Union (IoU) of the current bounding box with the existing boxes in the part selection set $\mc{P}$. By repeating this process, all the legal part proposals will be automatically added into the final selection set until $|\mc{P}| \ge N$. Specially, we further apply an iteration terminator if all this part selection module could not find appropriate object parts.
The detailed process is elaborated in Algorithm~\ref{alg:part}. With this proposed part selection module, our approach has the potential to find discriminative regional features automatically.

\subsection{Relation Embedding with Transformers}\label{sec:relationembed}
As mentioned in~\secref{sec:formulate}, one crucial problem in fine-grained recognition is to build contextual dependencies for different semantic parts.
However, limited by the design of 2D Convolutional Neural Networks, each feature unit has a restricted receptive field which is fixed by the convolutional kernels. To solve this problem, we propose to adopt the Transformer~\cite{vaswani2017attention}, which has been established as a state-of-the-art transduction model to take advantage of long-term correlation in semantic sentence understanding. Motivated by its succusses in natural language processing, we make an attempt to introduce it into the fine-grained recognition task, which is presented as the transformation operation in Eqn.~\eqref{eq:partform}.

\textbf{Fine-grained recognition with visual transformers.} The conventional transformers tend to take the input of a whole language sentence and building correlations of different semantic words, which are now used to replace the Recurrent Neural Networks (RNNs) in many NLP tasks to build long-term dependencies. Besides its superior performance in relation to modeling, embedding this architecture into existing CNN backbones,~\eg, ResNet, VGGNet, is thus essential. In our PART framework, CNN backbones and transformers play different roles in fine-grained visual understanding. Solely adopting transformer models may result in a larger receptive field in relational modeling, but leads to a high computation burden for network optimization. In addition, keeping the main body of CNNs can also benefit from the pretrained common knowledge,~\eg, ImageNet~\cite{deng2009imagenet}. Hence in our approach, we propose to make a conjunction of conventional CNNs and the successful Transformer model. The detailed architecture is illustrated in~\figref{fig:transformer}.

\textbf{Positional encoding.}
Another drawback in attention-based modeling methods is that structural information would be forgotten, especially when transforming CNN feature maps $\br{X}\in \mathbb{R}^{W \times H \times C}$ to multiple vectors. When dense connections are constructed, their original structural information could be easily lost.
Inspired by~\cite{vaswani2017attention,carion2020end} in keeping position information, here we adopt the positional encoding $\br{E}$ into the original backbone features $\br{X}$. With the positional encoded features regardless of part proposals in~\figref{fig:transformer}, the key idea is to build global relation matrices $\br{A} \in \mathbb{R}^{WH \times WH}$, indicating the dense dependencies of each pixel unit to all the other pixels. Let $i, j$ be the positions of query and key values, this spatial correlation can be formulated as:
\begin{equation}\label{eq:attention}
\br{A}_{i,j} = (\br{X}_{i,:} +  \br{E}_{i,:}) \br{W}_{qry} \br{W}_{key}^{\top} (\br{X}_{j,:}+\br{E}_{j,:})^{\top},
\end{equation}
where $\br{W}_{\{qry,key\}}$ denote the learnable weight of query and key projections. Here we adopt the relative positional encoding of Transformer-XL~\cite{dai2019transformer} to build different positional encoding for different pixels. We build the 1D learning for relative encoding and 2D positional learning for absolute encoding.  By expanding Eqn.~\eqref{eq:attention} and using learnable $\br{u},\br{v}$, this relative encoding correlations $\br{A}$ can be presented as:
\begin{equation}\label{eq:relative}
\begin{split}
\br{A}_{i,j}&= \br{X}_{i,:}^{\top} \br{W}_{qry}^{\top} \br{W}_{key}\br{X}_{j,:}+ \br{X}_{i,:}^{\top}\br{W}_{qry}^{\top} \widehat{\br{W}}_{key}\br{R}_{i-j}\\
& + \br{u}^{\top}\br{W}_{key}^{\top}\br{X}_{j,:}+ \br{v}^{\top}\widehat{\br{W}}_{key}\br{R}_{i-j},
\end{split}
\end{equation}
where key weights are spilt into $\br{W}_{key}$ and relative position $\widehat{\br{W}}_{key}$ respectively. $\br{R}$ denotes the sinusoid encoding matrix in~\cite{vaswani2017attention} without learnable weights.

\textbf{Multi-head attention with transformers.}
The relational matrices $\br{A}_{i,j}$ are multiplied with the original feature $\br{X}$, resulting in the enhanced feature of one head $\br{H}$:
\begin{equation}\label{eq:onehead}
\br{H}= \frac{\exp(\br{A}_{:} / \sqrt{C} )} {\sum_{j=1}^{WH} \exp(\br{A}_{:,j} / \sqrt{C} ) } \br{X} \br{W}_{val}.
\end{equation}
As illustrated in~\figref{fig:transformer}, each transformer unit is composed of multiple transformer heads,~\ie, $\mc{H} = \{ \br{H}_1,\ldots,\br{H}_m\}$ and $m$ denotes the number of attention heads. Employing multiple attention heads encourages different projections of query, key and value triplets to construct relational embedding. Then these attention heads are concatenated and fused with the feed-forward network, and the final output $\br{O}$ can be formally presented as:
\begin{equation}\label{eq:multihead}
\br{O}= \bs{Concat}(\br{H}_1,\ldots,\br{H}_m)\br{W}^o + \br{b}^o,
\end{equation}
where $\br{W}^o \in \mathbb{R}^{mC_h \times C} $ denotes the learnable weight for linear transformations, and $C_h$ denotes the dimension of one attention head.

\begin{figure}[!t]
	\centering
	\includegraphics[width=\columnwidth]{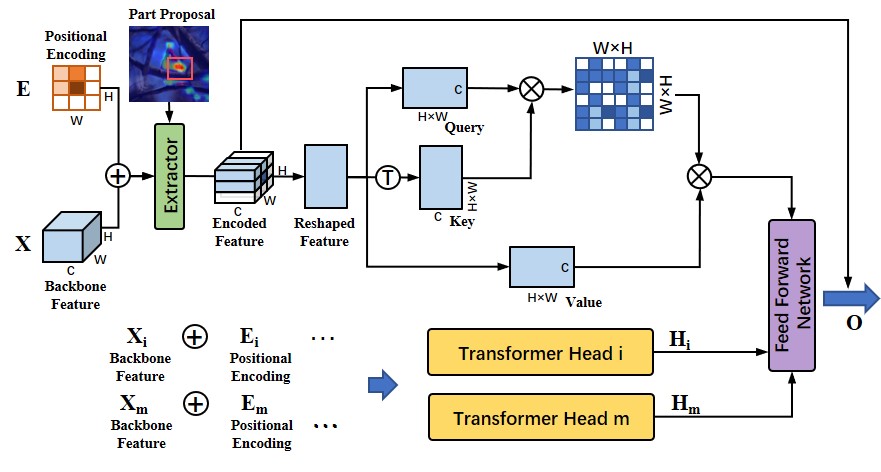}
	\caption{Illustration of the part-guided relational transformation module. When input backbone features and positional encoding information, this module builds local transformation embedding with efficient self-attention mechanisms with a multi-head ($\br{H}_1 \ldots \br{H}_m$) encoder manner.
	}
	\label{fig:transformer}
\end{figure}

\textbf{Relations with CNNs.}
Based on this designment of relational embedding transformers, networks have the potential to learn the long-term dependencies between any two pixels. Notably, this attentive transformer layer and convolutional layers process 2D features in similar ways. In other words, the transformer layers are attached with more relational embedding learning matrices compared to the convolutional layers and in extreme cases, the transformers can be degenerated to simulate CNNs. The main theorem from~\cite{cordonnier2020relationship} is as follows:
\newtheorem{thm}{\bf Theorem}[section]
\begin{thm}\label{thm1}
A transformer layer with $m$ heads and with hidden dimension $C_h$ and output dimension $C_{out}$ can simulate convolutional neural networks with $\min(C_h, C_{out})$ channels.
\end{thm}

Detailed proof can refer to~\cite{cordonnier2020relationship}, which is abstracted in Appendix~\ref{sec:appendix}. Theorem~\ref{thm1} shows that the transformer modules have the potential to represent existing CNN layers but have the strong potential to explore cross-pixel relations. It means that the CNNs are ``restricted" transformers that have fixed attentional embedding or limited receptive fields.

\textbf{Part-guided relation embedding.} As aforementioned in Eqn.~\eqref{eq:partform}, our proposed optimization framework aims to construct global relations as well as discovering discriminative part relations simultaneously. Benefiting from the strong relation embedding ability of transformer models, we propose to learn the part-level relations by the auto-discovered part in~\secref{sec:partdis}. We first adopt the learned part masks $\br{p}_i \in \mc{P}$ as attentive masks before feeding into the transformer layers.
Thus the learned architecture of each attention head can be presented as:
\begin{equation}\label{eq:parthead}
\br{H}^{\br{p}_i}= \bs{Attention}( (\br{E}(\br{p}_i)+\br{X}_{i,:}\odot \br{p}_i) ),
\end{equation}
where $\odot$ denotes the Hadamard product for mask region selection. We thus concatenate the multiple heads of each part following Eqn.~\eqref{eq:multihead}, forming output $\br{O}^{\br{p}_i}$ of each part $\br{p}_i$. After passing the mask selection operation, a dense connection within each local part is constructed, which discovers the spatial relations while neglects the unnecessary ones.

Moreover, the part-guided transformer units can be stacked in sequence, passing each enhanced feature $\br{O}^{\br{p}_i}$ into the next layer. This operation enlarges the reception field of each pixel and constructs strong correlations by mixing connections hops:
\begin{equation}\label{eq:stacktrans}
\br{O}^{\br{p}_i}_{t}= \mc{T}^{\br{p}_i}_{t-1}(\br{O}^{\br{p}_i}_{t-1},\br{E}_{t}+\br{X},\br{p}_i), t=1\ldots S,
\end{equation}
where $S$ denotes the number of stacked layers of transformer model $\mc{T}$. Similarly, the global relation embedding transformer $\mc{G}$ is defined in similar manners.

In our training scheme, we conduct a collaborative learning scheme of the global and $N$ part transformation branches in~\figref{fig:pipeline}. Each branch is supervised with an individual softmax cross-entropy loss for minimizing target $\mathbb{E}$ in Eqn.~\eqref{eq:partform}. Thus our overall learning is composed of one global loss for contextual relation learning and $N$ part losses for local feature discovery. Note that our proposed PART approach resort to the $N$ auxiliary part branches for feature regularization in the training stage, whilst not relying on these parts in the testing stage. Moreover, simply aggregating these features would also lead to overfitting issues and inferior results. Compared to the ResNet backbone, we only introduce limited learnable matrices for relational embedding and conduct an end-to-end training with the conjunction of part discovery module and relational embedding module. Our framework greatly alleviates the computation burden and retains the discriminative part simultaneously.

\section{Experiments}\label{sec:experiment}
In this section, we first introduce the experimental settings with dataset statistics in~\secref{sec:expsettings} and elaborate the implementation details and network architectures in~\secref{sec:expimplementaion}. Subsequently, the comparison with state-of-the-art methods is exhibited in~\secref{sec:expcomp}. The detailed performance analyses, as well as the explainable visualizations, are presented in~\secref{sec:expperformance}.

\subsection{Experimental Settings}\label{sec:expsettings}
To evaluate the effectiveness of our proposed PART approach, we conduct experiments on three public popular benchmarks, namely Caltech-UCSD Birds (CUB-200-2011)~\cite{wah2011caltech}, Stanford-Cars~\cite{krause20133d}, and FGVC-Aircrafts~\cite{maji2013fine}. CUB-200-2011 dataset~\cite{wah2011caltech} contains 11,788 images of 200 wild bird species, which is the most widely used benchmark. Stanford-Cars dataset~\cite{krause20133d} includes 16,185 images of 196 car subcategories. FGVC-Aircraft dataset~\cite{maji2013fine} contains 10,000 images of 100 classes of fine-grained aircrafts.
We follow the standard training/testing splits as in the original works~\cite{wah2011caltech,krause20133d,maji2013fine}. For evaluations, we use the top-1 evaluation criterion following previous works and only adopt the classification label for supervised training without any additional part annotations.

\begin{table}[t]
\centering{
\caption{Comparison with state-of-the-art results on CUB-200-2011 dataset.  $\dag$: using additional annotations. }
\label{table:CUB}
\renewcommand{\arraystretch}{1.2}
\begin{tabular}{c|c|c|c}
\toprule
Type&Method&Backbone& Accuracy \\
\midrule
\multirow{7}*{\tabincell{c}{Part\\Based}}
&PA-CNN$\dag$~\cite{krause2015fine}&VGG-19&84.3\%\\
&FCAN$\dag$~\cite{liu2016fully}&ResNet50&84.7\%\\
&RA-CNN~\cite{fu2017look}&VGG-19&85.3\%\\
&MA-CNN~\cite{zheng2017learning}&VGG-19&86.5\%\\
&Interpret~\cite{huang2020interpretable}&ResNet101&87.3\%\\
&NTS-Net~\cite{yang2018learning}&ResNet50&87.5\%\\
&DF-GMM~\cite{wang2020weakly}&ResNet50&88.8\%\\
\midrule
\multirow{13}*{\tabincell{c}{Feature\\Based}}
&Bilinear~\cite{lin2015bilinear}&VGG-16&84.0\%\\
&MAMC~\cite{sun2018multi}&ResNet101&86.5\%\\
&MaxEnt~\cite{dubey2018maximum}&DenseNet-161&86.5\%\\
&DFL-CNN~\cite{wang2018learning}&VGG-16&86.7\%\\
&PC~\cite{dubey2018pairwise}&DenseNet-161&86.9\%\\
&HBP~\cite{yu2018hierarchical}&VGG-16&87.1\%\\
&DFL-CNN~\cite{wang2018learning}&ResNet50&87.4\%\\
&Cross-X~\cite{luo2019cross}&ResNet50&87.7\%\\
&DCL~\cite{chen2019destruction}&ResNet50&87.8\%\\
&TASN~\cite{zheng2019looking}&ResNet50&87.9\%\\
&ACNet~\cite{ji2020attention}&ResNet50&88.1\%\\
&API-Net~\cite{zhang2020learning}&ResNet101&88.6\%\\
&PMG~\cite{du2020fine}&ResNet50&\textbf{89.6\%}\\
\midrule
\multirow{2}*{\tabincell{c}{Joint}}
&PART (Ours)&ResNet50&\textbf{89.6\%}\\
&PART (Ours)&ResNet101&\textbf{90.1\%}\\
\bottomrule
\end{tabular}
}
\end{table}

\subsection{Implementation Details}\label{sec:expimplementaion}
For the implementation of baseline and our proposed PART framework, we adopt ResNet-50 and ResNet-101~\cite{he2016deep} networks pretrained on ImageNet~\cite{deng2009imagenet} as our backbone for fair comparisons. We remove the last bottleneck in stage 4 to directly get the feature of 512 dimensions. We use the SGD optimizer with an initial learning rate of $8e-4$ annealed by 0.1 for every 60 epochs. We adopt the commonly used data augmentation techniques,~\ie, random cropping and erasing, left-right flipping, and color jittering for robust feature representations. Our model is relatively lightweight and is trained end-to-end on two NVIDIA 2080Ti GPUs for acceleration. We use the group-wise sampler with a batchsize of 16, sampling 4 samples for each class.
For hyper-parameter fine-tuning, we randomly select 10\% of the training set as validation.
The balanced weights $\lambda_{\br{p}}$ for multiple part loss functions are empirically set as 0.1.
The channel dimension for feeding into relation embedding module is $C=512$ and the part branch number is set as 4, the IoU threshold is $th=0.6$. The stacked layer number $S$ is set as 3 for the global branch and 1 for the local branches. The training and testing protocol follow the state-of-the-art works~\cite{luo2019cross,chen2019destruction,zhang2020learning} using random cropping of $448\times448$ in training and only use one-crop during inference, which does not rely on the multi-crop features like previous works~\cite{fu2017look,zheng2017learning,yang2018learning}. Notably, in the inference stage, we only use the global feature for accuracy reports, and the local branches can be simply omitted for acceleration.

\subsection{Comparison with State of The Arts}\label{sec:expcomp}
In this subsection, we conduct experimental comparisons with state-of-the-art models on three representative benchmarks,~\ie, CUB-200-2011~\cite{wah2011caltech}, Stanford-Cars~\cite{krause20133d}, and FGVC-Aircrafts~\cite{maji2013fine}.

\textbf{Comparison on CUB-200-2011 dataset.} CUB-200-2011 is the most widely-used dataset, consisting of 200 bird species with visually similar appearances. The classification accuracy can be found in~\tabref{table:CUB}. Following the aforementioned group ways in~\secref{sec:intro}, we divide the state-of-the-art models into two groups,~\ie, feature regularized learning and part-based learning. It can be found that the earlier works,~\eg,~\cite{krause2015fine,liu2016fully} relies on the accurate part localization information, while achieving only 84.7\% base performance. Hence with the development of auto-discovered part localization methods~\cite{zheng2017learning,yang2018learning}, the classification accuracy has been increased by over 3\% without any additional annotations.

On the other hand, feature-learning-based algorithms also show advantages in extracting informative knowledge and forming efficient representations. Hence our proposed PART approach is a joint learning algorithm with collaborative feature relation learning and part feature discovery. Combining these merits, our proposed method generates state-of-the-art results of 89.6\% accuracy, which demonstrates the effectiveness of our proposed learning framework.
Following~API-Net~\cite{zhang2020learning}, we extend our model with the deeper ResNet101 backbones, which achieves 90.1\% and is 1.5\% higher than~\cite{zhang2020learning}.

\begin{table}[t]
\centering{
\caption{Comparison with state-of-the-art results on FGVC-Aircraft dataset.}
\label{table:Aircraft}
\renewcommand{\arraystretch}{1.2}
\begin{tabular}{c|c|c|c}
\toprule
Type&Method&Backbone& Accuracy \\
\midrule
\multirow{3}*{\tabincell{c}{Part\\Based}}
&RA-CNN~\cite{fu2017look}&VGG-19&88.2\%\\
&MA-CNN~\cite{zheng2017learning}&VGG-19&89.9\%\\
&NTSNet~\cite{yang2018learning}&ResNet50&91.4\%\\
\midrule
\multirow{14}*{\tabincell{c}{Feature\\Based}}
&Bilinear~\cite{lin2015bilinear}&VGG-16&84.1\%\\
&PC~\cite{dubey2018pairwise}&DenseNet-161&89.2\%\\
&MaxEnt~\cite{dubey2018maximum}&DenseNet-161&89.8\%\\
&HBP~\cite{yu2018hierarchical}&VGG-16&90.3\%\\
&ACNet~\cite{ji2020attention}&VGG-16&91.5\%\\
&DFL-CNN~\cite{wang2018learning}&ResNet50&92.0\%\\
&ACNet~\cite{ji2020attention}&ResNet50&92.4\%\\
&CIN~\cite{gao2020channel}&ResNet50&92.6\%\\
&Cross-X~\cite{luo2019cross}&ResNet50&92.7\%\\
&CIN~\cite{gao2020channel}&ResNet101&92.8\%\\
&DCL~\cite{chen2019destruction}&ResNet50&93.0\%\\
&API-Net~\cite{zhang2020learning}&ResNet50&93.0\%\\
&PMG~\cite{du2020fine}&ResNet50&93.4\%\\
&API-Net~\cite{zhang2020learning}&ResNet101&93.4\%\\
\midrule
\multirow{2}*{\tabincell{c}{Joint}}
&PART (Ours)&ResNet50&\textbf{94.4\%}\\
&PART (Ours)&ResNet101&\textbf{94.6\%}\\
\bottomrule
\end{tabular}
}
\end{table}

\textbf{Comparison on FGVC-Aircraft dataset.} Similar to the benchmarking on CUB-200-2011 dataset, here we make comparisons on sub-categories of aircrafts in~\tabref{table:Aircraft}. Recent approaches,~\eg, DCL~\cite{chen2019destruction} and Cross-X~\cite{luo2019cross}, achieve performance of $93.0\%$ and $92.7\%$, which is much higher than the previous work~\cite{lin2015bilinear}. To make fair comparisons with these works, we adopt the lightweight ResNet-50 backbone to achieve the state-of-the-art performance of 94.4\%, while previous works rely on deep backbones for feature extraction,~\eg,~\cite{gao2020channel} with ResNet-101 and~\cite{dubey2018pairwise,dubey2018maximum} with DenseNet-161. In addition, other works~\cite{fu2017look,zheng2017learning,yang2018learning} adopt the multi-crop inference to obtain multi-scale features, introducing additional computation burdens. Different from these mentioned issues, our proposed PART approach presents robust features with ResNet-50 backbone and one-crop inference.

\textbf{Comparison on Stanford-Cars dataset.}~\tabref{table:Cars} exhibits the results on the Stanford-Cars dataset. Stanford-Cars is a much easier dataset, in which previous methods achieve preliminary results of over 92.5\% accuracy. It can be found that methods on this dataset perform very similar performance,~\eg, TASN~\cite{zheng2019looking} and HBP~\cite{yu2018hierarchical} of 93.7\% performance. While obtaining multi-level feature representations~\cite{luo2019cross} and using separate layer initialization strategies~\cite{wang2018learning}, high performance as 94.5\% can be achieved. Surprisingly, even on this dataset, our method can provide a steady improvement compared to state-of-the-art results, reaching 95.1\% performance. With deep backbones~\ie, ResNet101, the performance shows a slight improvement to 95.3\% in top-1 accuracy.

\begin{table}[t]
\centering{
\caption{Comparison with state-of-the-art results on Stanford-Cars dataset. $\dag$: additional bounding box or segmentation annotation.}
\label{table:Cars}
\renewcommand{\arraystretch}{1.2}
\begin{tabular}{c|c|c|c}
\toprule
Type&Method&Backbone& Accuracy \\
\midrule
\multirow{3}*{\tabincell{c}{Part\\Based}}
&PA-CNN$\dag$~\cite{krause2015fine}&VGG-19&92.8\%\\
&MA-CNN~\cite{zheng2017learning}&VGG-19&92.8\%\\
&NTSNet~\cite{yang2018learning}&ResNet50&93.9\%\\
\midrule
\multirow{15}*{\tabincell{c}{Feature\\Based}}
&Bilinear~\cite{lin2015bilinear}&VGG-16&91.3\%\\
&Grassmann~\cite{wei2018grassmann}&VGG-16&92.8\%\\
&PC~\cite{dubey2018pairwise}&DenseNet-161&92.9\%\\
&MaxEnt~\cite{dubey2018maximum}&DenseNet-161&93.0\%\\
&MAMC~\cite{sun2018multi}&ResNet101&93.0\%\\
&HBP~\cite{yu2018hierarchical}&VGG-16&93.7\%\\
&TASN~\cite{zheng2019looking}&ResNet50&93.7\%\\
&DFL-CNN~\cite{wang2018learning}&ResNet50&93.8\%\\
&CIN~\cite{gao2020channel}&ResNet50&94.1\%\\
&CIN~\cite{gao2020channel}&ResNet101&94.5\%\\
&Cross-X~\cite{luo2019cross}&ResNet50&94.5\%\\
&DCL~\cite{chen2019destruction}&ResNet50&94.5\%\\
&API-Net~\cite{zhang2020learning}&ResNet50&94.8\%\\
&API-Net~\cite{zhang2020learning}&ResNet101&95.1\%\\
&PMG~\cite{du2020fine}&ResNet50&\textbf{95.1\%}\\
\midrule
\multirow{2}*{\tabincell{c}{Joint}}
&PART (Ours)&ResNet50&\textbf{95.1\%}\\
&PART (Ours)&ResNet101&\textbf{95.3\%}\\
\bottomrule
\end{tabular}
}
\end{table}

\begin{figure*}
\begin{center}
\includegraphics[width=1\textwidth]{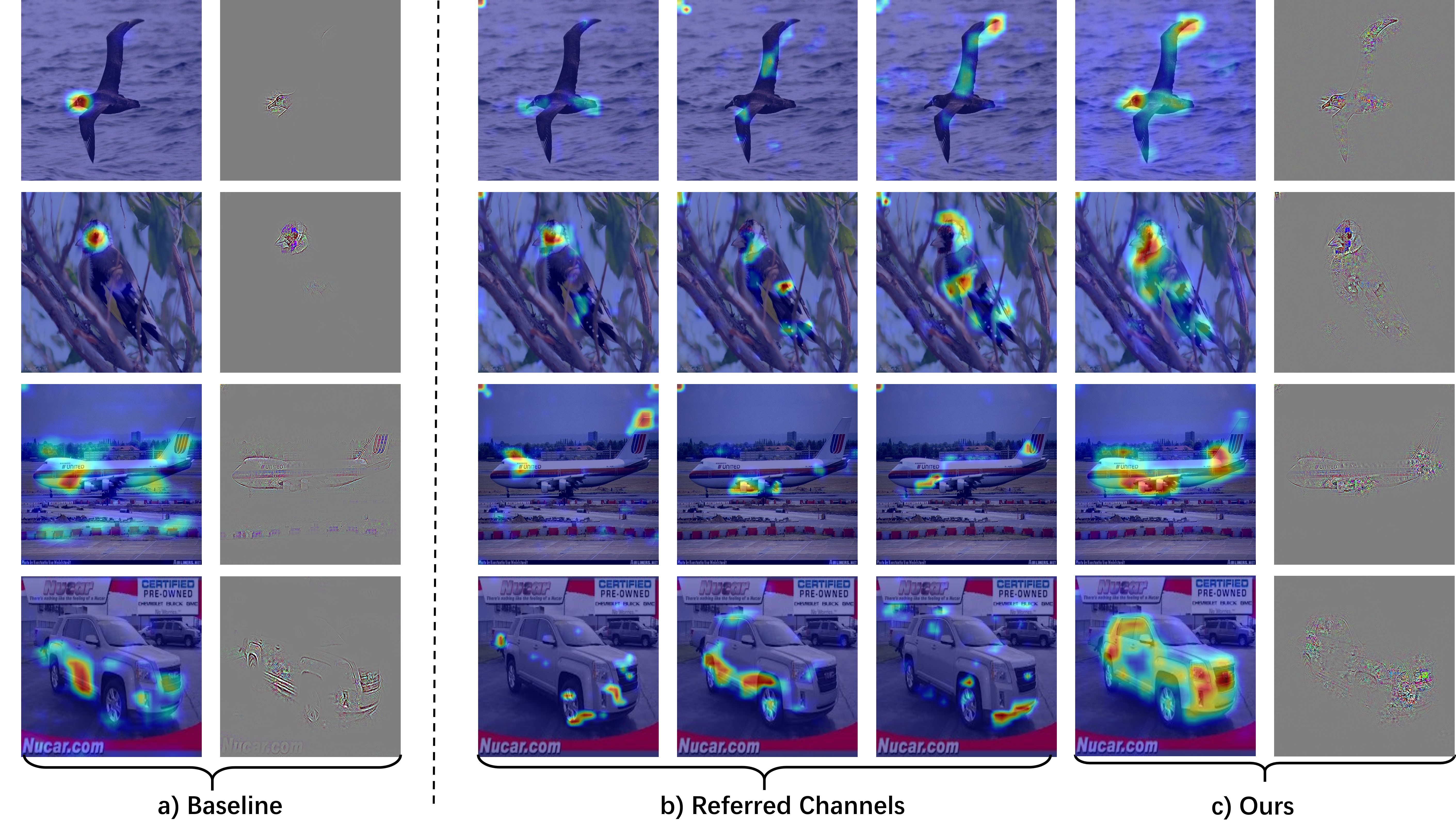}
 \caption{Class activation visualizations of baseline (ResNet50) and our proposed PART. Left figures in a) and c) are the weighted summed class activation map of all channels. right figures in a) and c) are the guided gradient of baseline and our model generated by Grad-CAMs~\cite{selvaraju2017grad}. Images in b) are referred channels generated by our proposed PART.}
 \label{fig:experiment}
 \end{center}
\end{figure*}

\subsection{Performance Analysis}\label{sec:expperformance}

In this subsection, we first conduct ablation studies of our proposed different modules, and study the manner of positional encoding in relation to transformation modules. Then we make a detailed analysis of the part effects in our framework and finally make explanations using visualized comparisons.

\begin{table}[t]
\centering{
\caption{Ablation studies of our different components on three benchmarks. $\mc{M}_{Relation}$, $\mc{M}_{Part}$ denotes the proposed relation transformation module and part-discovery module respectively. Quaternion $(x,y_1,y_2,y_3)$ indicates the head number of global transformer is $x$ and heads of three parts are $y_1,y_2$ and $y_3$.}
\label{table:ablation}
\setlength{\tabcolsep}{1.2mm}
\renewcommand{\arraystretch}{1.2}
\begin{tabular}{cc|c|c|c}
\toprule
$\mc{M}_{Relation}$&$\mc{M}_{Part}$  & CUB-200-2011& FGVC-Aircraft& Stanford-Cars \\
\midrule
N/A&N/A&85.4\%&90.7\%&93.4\%\\
(2,0,0,0)&N/A&88.2\%&93.5\%&94.7\%\\
(2,1,1,1)&$\times 3$&89.2\%&94.2\%&\textbf{95.1\%}\\
(4,1,1,1)&$\times 3$&89.3\%&94.3\%&\textbf{95.1\%}\\
(4,1,1,1)&$\times 4$&\textbf{89.6\%}&\textbf{94.4\%}&\textbf{95.1\%}\\
\bottomrule
\end{tabular}
}
\end{table}

\textbf{Ablation studies on different components.} To evaluate the effectiveness of our proposed part discovery module and relation transformation module, we conduct detailed ablation studies on three representative benchmarks~\cite{wah2011caltech,krause20133d,maji2013fine}. In~\tabref{table:ablation}, the first line shows the baseline performance on these three benchmarks, using the same data augmentation and training hyper-parameters as our final model. It can be found that our baseline model surpasses many earlier works and our proposed modules,~\ie, relation transformation module $\mc{M}_{Relation}$ and part discovery module $\mc{M}_{Part}$ can improve the recognition ability steadily,~\eg, from 85.4\% to 89.6\% (4 parts) on CUB-200-2011 dataset.
In the second row, we only adopt the relation transformation module with stacked manners, and set the head number of multi-head attention as 2. It can be found that the performance has increased notably by incorporating the part-level features. This verifies that the effectiveness of our proposed module, even on the high-performance baseline. Moreover, when setting the number of attention heads as 4, performance on these three datasets can be slightly improved.

\begin{table}[t]
\centering{
\caption{Performance analysis of feature extraction and feature encoding manner on CUB-200-2011 benchmark. $\mc{M}_{Relation}$, $\mc{M}_{Part}$ denotes the proposed relation transformation module and part-discovery module respectively.}
\label{table:encoding}
\renewcommand{\arraystretch}{1.2}
\begin{tabular}{lc|c}
\toprule
Feature Extraction& Encoding Manner &Accuracy\\
\midrule
Baseline (w/o group sampler)&N/A&84.2\%\\
Baseline&N/A&85.4\%\\
$\mc{M}_{Relation}$&N/A&87.5\%\\
$\mc{M}_{Relation}$&Absolute Encoding&88.1\%\\
$\mc{M}_{Relation}$&Learnable Encoding&88.2\%\\
$\mc{M}_{Relation}+\mc{M}_{Part}$&Learnable Encoding&89.2\%\\
\bottomrule
\end{tabular}
}
\end{table}

\textbf{Effects of feature encoding.} We first present the performance of the baseline method in the first two rows on the CUB benchmark dataset. The first row indicates that without the group sampler in~\secref{sec:expimplementaion} of a class-wise balanced manner, baseline model would result in a performance drop of over $1\%$.
As a natural drawback in building contextual dense relations, the spatial structure information is lost during the vectorization process. The result (with 2 attention heads) without positional information can be found in the third-row of~\tabref{table:encoding}. By incorporating the absolute positional encoding operation (sinusoid encoding function in~\cite{vaswani2017attention}), the recognition accuracy has been increased from 87.5\% to 88.1\%. In this paper, we adopt relative encoding which uses learned encoding parameters to encode the unique positional information for each pixel. The result can be found in the fifth row, which also has the potential to simulate any convolutional layers as in Theorem~\ref{thm1}. And finally our full model with the part discovery module achieves the best performance in the last row.

\begin{table}[t]
\centering{
\caption{Hyper parameter experiments on part discovery module on CUB-200-2011 benchmark. $|\br{X}^{*}|$ and $|\mc{P}|$ denotes the number of parts in the part stack and final selected parts respectively. $\mc{H}$ indicates the head number of transformers. }
\label{table:hyper}
\setlength{\tabcolsep}{1.3mm}
\renewcommand{\arraystretch}{1.1}
\begin{tabular}{c|c|c|c|c|c|c|c|c}
\toprule
$|\br{X}^{*}|$&N/A&64&64&64&32&64&64&64\\
$|\mc{P}|$&N/A&1&2&3&3&3&4&5\\
$\mc{H}$&2&2&2&2&2&4&4&4\\
\midrule
Acc.&88.2\%&88.8\%&89.0\%&89.2\%&89.1\%&89.3\%&89.6\%&88.8\%\\
\bottomrule
\end{tabular}
}
\end{table}

\textbf{Effects of part discovery.} To evaluate the effectiveness of the part discovery module, here we exhibit different hyperparameters in the part discovery module in~\tabref{table:hyper}. In our proposed algorithm, the part selection is mainly affected by two hyper-parameters,~\ie, the capacity of the part stack
$|\br{X}^{*}|$ and the final selected part number $|\mc{P}|$. The selected part number also affects the network architecture, deciding the number of regularized part-level branches. When applying no auxiliary part branches, in~\tabref{table:hyper}, our proposed approach with only transformer modules achieves the performance of 88.2\%. With the increase of selected part number $|\mc{P}|$, auxiliary supervision is also automatically added into the overall framework. It can be found that our approach with more parts ($|\mc{P}|=4$) reaches a high performance of 89.6\%. This indicates the effectiveness of using auxiliary part branches as regularization. However, with the part numbers $|\mc{P}|$ increasing to be 5, the final performance decreases to 88.8\%. Discovering more parts would incorporate meaningless background regions into computation and also restrain the learning of global features, which jointly lead to an inferior response on class activation maps.
Moreover, when squeezing the capacity of the part stack $|\br{X}^{*}|$, the part discovery module would also lead to similar performance, which verifies the robustness of our proposed method.

During the inference phase, we only rely on the global branches and drop the redundant part features. We also conduct experiments of direct concatenating these branches rather than using them as regularization. The performance of condition $|\mc{P}|=1$ in the first row of~\tabref{table:hyper} drops to 88.5\%.
Assuming the same classification task of 200 classes, the overall parameters of ResNet50 is 25.6M. Comparing to the full model with all 4 parts of 32.9M, our final model only requires 28.7M parameters for learning, which is implementation-friendly.

\textbf{Visualized explanations.} We further investigate the class activation maps using state-of-the-art Grad-CAMs~\cite{selvaraju2017grad}. As shown in~\figref{fig:experiment}, we present the summation of all channels of baseline model (ResNet50) in~\figref{fig:experiment} a), and our proposed PART in c). Interestingly, the baseline model can easily fall into local attentions to distinguish the visually similar samples, which leads to inferior generalization capabilities. While our proposed PART can localize the full object, resulting in a robust recognition ability. Moreover, we visualize the referred channels of our model in~\figref{fig:experiment} b). It can be found that different channels focus on different semantic parts, forming the final robust representations. Although these ``parts" do not strictly meet the definition of natural semantics, while still responding to meaningful object regions.

\section{Conclusions}\label{sec:conclusion}
In this paper, we propose a novel PArt-guided Relational Transformers (PART) framework for fine-grained recognition tasks. To solve the deficiencies in building long-term relationships in conventional CNNs, we make the first attempt to introduce the Transformer architecture into fine-grained visual recognition tasks. Beyond these insights, we further present a new part discovery module to automatically mine discriminative regions by utilizing the class activation maps during the training process. With these generated part regions, we propose a part-guided transformation module to learn high-order spatial relationships among semantic pixels. Our full model is composed of several local regional transformers and one global transformer, enhancing the discriminative regions while maintaining the contextual one. Experimental results verify the effectiveness of our proposed modules and our proposed PART reaches new state-of-the-art on 3 widely-used fine-grained recognition benchmarks.


% or



\section*{Acknowledgment}
This work was supported by grants from National Natural Science Foundation of China (No.61922006 and No.61825101).

\ifCLASSOPTIONcaptionsoff
  \newpage
\fi

%\bibliographystyle{IEEEtran}
%\bibliography{IEEEabrv,../IQA}

\bibliographystyle{IEEEtran}
\bibliography{FGVC}

\appendix[Proof of the Theorem~\ref{thm1}] \label{sec:appendix}
This Appendix provides a brief proof of theorem~\ref{thm1}, which is derived from previous research~\cite{cordonnier2020relationship} and not the main contribution of our paper.

\begin{thm}
A transformer layer with $m$ heads and with hidden dimension $C_h$ and output dimension $C_{out}$ can simulate convolutional neural networks with $\min(C_h, C_{out})$ channels.
\end{thm}

\begin{proof}
A typical transformer layer $\mc{T}$ is composed of multi-head attention ($\bs{MHA}$) layers with $m$ heads and several feed forward layers (multi-layer perceptron). Considering Eqn.~\eqref{eq:onehead} and Eqn.~\eqref{eq:multihead}, it can be formally presented as:
\begin{equation}\label{eq:proof1}
\begin{split}
\mc{T} &= ((\bs{MHA}(\br{X})\br{X} \br{W}_{val}+\br{b}^{o})\br{W}^{f}+\br{b}^{f} \\
&=(\bs{Concat}_{h=1}^{m} (\sum_{k} \varphi (\br{A}^{(h)})_{k} \br{X}_{k}) \br{W}^{h}_{val}+\br{b}^{o})\br{W}^{f}+\br{b}^{f},
\end{split}
\end{equation}
where $\varphi$ denotes the $\bs{softmax}$ operation and $\br{W}^{f},\br{b}^{f}$ are learnable weights for feed-forward layers. In this equation, the learnable weights $\br{W}^{h}_{val}$ and $\br{W}^{f}$ can be formulated into one unified projection of $\br{W}^{h}_{proj}$.
In this manner, Eqn.~\eqref{eq:proof1} can be further reformulated as:
\begin{equation}\label{eq:proof2}
\mc{T} =\bs{Concat}_{h=1}^{m} (\sum_{k} \varphi (\br{A}^{(h)})_{k} \br{X}_{k}) \br{W}^{h}_{proj}+\br{b},
\end{equation}
Considering a convolutional layer $\bs{Conv}(\br{X})=\br{X}_{k} \br{W}+\br{b}$, the transformer layer would degenerated to a convolutional layer, when the attention value $ \varphi (\br{A}^{(h)})_{k} $ equals to $1$, when the positional biases $\Delta_{\{1,2\}}$ of convolutional kernel centers equals to the attention biases $k-q$. Considering the relative encoding in the main manuscript:
\begin{equation}\label{eq:proof3}
\begin{split}
\br{A}_{k,q}&= \br{X}_{k,:}^{\top} \br{W}_{qry}^{\top} \br{W}_{key}\br{X}_{q,:}+ \br{X}_{k,:}^{\top}\br{W}_{qry}^{\top} \widehat{\br{W}}_{key}\br{R}_{k-q}\\
& + \br{u}^{\top}\br{W}_{key}^{\top}\br{X}_{q,:}+ \br{v}^{\top}\widehat{\br{W}}_{key}\br{R}_{k-q},
\end{split}
\end{equation}

As proved by Cordonnier~\etal~\cite{cordonnier2020relationship}, there exists an attention matrix $ \varphi (\br{A}^{(h)})_{k} $ with relative encoding to be $1$ if the bias in convolutional kernels equals to $k-q$, otherwise $ \varphi (\br{A}^{(h)})_{k} =0$. To construct this, we set $\br{W}_{key} = \br{W}_{qry} = 0$, thus Eqn.~\eqref{eq:proof3} only has the last term of $\br{v}^{\top}\widehat{\br{W}}_{key}\br{R}_{k-q}$.

In the above expression, we set $\widehat{\br{W}}_{key} = \br{I}$ and assume $(\br{A}^{(h)}_{q})_{k}=\br{v}^{\top}\br{R}_{k-q}=-\alpha||k-q-\Delta||^2+\alpha c$, the softmax operation $ \varphi (\br{A}^{(h)}_{q})_{k} $ yields $1$. $\alpha$ and $c$ denotes the coefficient and constant respectively.
Note that $\Delta_{\{1,2\}} \in \mathbb{Z}^2$ denote the biases of x-axis and y-axis direction.

With these values fixed, the limitation of softmax value $\varphi (\br{A}^{(h)})_{k} $ could reach $1$. On the other conditions ($k-q \ne \Delta$), its limitation would be zero. Hence theorem~\ref{thm1} can be concluded.
\end{proof}

\balance
\end{document}